\documentclass[12pt]{article}
\usepackage{dsfont,amssymb,amsmath,amsthm}
\usepackage{graphicx, color}

\usepackage{geometry}
\usepackage{setspace}
\usepackage{enumerate}
\usepackage{epigraph}

\newtheorem{theorem}{Theorem}[section]
\newtheorem{lemma}[theorem]{Lemma}
\newtheorem{claim}[theorem]{Claim}

\usepackage[pdftex,pagebackref=true,colorlinks]{hyperref}

\newcommand{\R}{\mathbb R}

\newcommand{\eps}{\epsilon}

\newcommand{\VC}{\text{VC}}
\newcommand{\DD}{\text{DD}}

\begin{document}

\title{Sample compression schemes for VC classes}

\author{Shay Moran\thanks{Departments of Computer Science, Technion-IIT, Israel and Max Planck Institute for Informatics, Saarbr\"{u}cken, Germany. {\tt  shaymrn@cs.technion.ac.il.}
Research is supported by ISF and BSF.}
\and Amir Yehudayoff\thanks{Department of Mathematics, Technion-IIT, Israel.
{\tt amir.yehudayoff@gmail.com.}
Horev fellow -- supported by the Taub foundation.  
Research is also supported by ISF and BSF.}}

\date{}

\maketitle

\begin{abstract}
Sample compression schemes were defined by Littlestone and Warmuth (1986)
as an abstraction of the structure underlying many learning algorithms.
Roughly speaking, a sample compression scheme of size $k$ 
means that given an arbitrary list of labeled examples, 
one can retain only $k$ of them
in a way that allows to recover the labels of all other examples in the list.
They showed that compression implies PAC learnability 
for binary-labeled classes,
and asked whether the other direction holds.
We answer their question and show that every concept class $C$
with VC dimension $d$ has a sample compression scheme of size exponential in $d$.
The proof uses an approximate minimax phenomenon
for binary matrices of low VC dimension,
which may be of interest in the context of game theory.
\end{abstract}

\section{Introduction}

Learning and compression are known to be deeply related to each other.
Learning procedures perform compression,
and compression is an evidence of and is useful in learning.
For example, support vector machines,
which are commonly applied to solve classification problems,
perform compression (see Chapter 6 in~\cite{Cristianini00a}).
Another example is the use of 
compression to boost the accuracy of learning procedures (see~\cite{littleWarm,DBLP:journals/iandc/Freund95} and Chapter 4 in~\cite{schapire2012boosting}).

About thirty years ago, Littlestone and Warmuth~\cite{littleWarm} 
provided a mathematical framework for studying compression
in the context of learning theory.
In a nutshell, they showed that compression indeed implies learnability
and asked whether learnability implies compression.

\subsection{Learning}

Here we provide a brief description of standard learning terminology. 
For more information, see the books \cite{KearnsVazirani94,
schapire2012boosting,Cristianini00a}.

Imagine a student who wishes to learn a concept
$c : X \to \{0,1\}$
by observing some training examples.
In order to eliminate measurability issues, 
we focus on the case that $X$ is a finite or countable set
(although the arguments we use are more general).
The high level goal of the student is to come up with an hypothesis
$h : X \to \{0,1\}$ that is close to the unknown concept $c$ 
using the least number of training examples.
There are many possible ways to formally 
define the student's objective.
An important one is 
Valiant's probably approximately correct (PAC) learning model~\cite{zbMATH03943062},
which is closely related to an earlier work of Vapnik and Chervonenkis~\cite{zbMATH03391742}.
This model is defined as follows.

The training examples are modeled as a pair
$(Y,y)$ where $Y \subseteq X$ is the multiset of points
the student observes and $y = c|_Y$ is their labels according to $c$.
The collection of all possible training examples
is defined as follows.
Let $C\subseteq \{0,1\}^X$ be a concept class.
A $C$-labeled sample is a pair $(Y,y)$,
where $Y \subseteq X$ is a multiset and $y = c|_Y$ for some $c \in C$. 
The size of a labeled sample $(Y,y)$ is 
the size of $Y$ as a multiset.
For an integer $k$, 
denote by $L_C(k)$ the set of $C$-labeled samples 
of size at most $k$.
Denote by $L_C(\infty)$ the set of all $C$-labeled samples of finite size.

The concept class $C$ is PAC learnable 
with $d$ samples, generalization error $\eps$,
and probability of success $1-\delta$ 
if there is a learning map $H:L_C(d) \to \{0,1\}^X$ so that
the hypothesis $H$ generates is accurate with high probability.
Formally,
for every $c \in C$ and for every probability distribution $\mu$ on $X$,
$$\Pr_{\mu^d} \Big[ \left\{ Y \in X^d : \mu(\{x \in X : h_Y(x) \neq c(x)\}) \leq \eps \right\} \Big] \geq 1-\delta ,$$
where $h_Y = H(Y,c|_Y)$.
In this text, when the parameters $\eps,\delta$ are not explicitly stated
we mean that their value is $1/3$.
If the image of $H$ is contained in $C$,
we say that $C$ is properly PAC learnable.

A fundamental question that emerges is
characterizing the sample complexity of PAC learning.
The work of Blumer, Eherenfeucht, Haussler, and Warmuth~\cite{zbMATH04143473},
which is based on~\cite{zbMATH03391742},
provides such a characterization.
The characterization is based on the Vapnik-Chervonenkis (VC) dimension of $C$,
which is defined as follows.
A set $Y \subseteq X$ is $C$-shattered if
for every $Z \subseteq Y$ there is $c \in C$ so that $c(x)=1$ for all $x \in Z$ and $c(x)=0$ for all $x \in Y-Z$.
The VC dimension of $C$, denoted $\VC(C)$, is the 
maximum size of a $C$-shattered set (it may be infinite).
They proved that the sample complexity of PAC learning
$C$ is $\VC(C)$, up to constant factors\footnote{Big $O$ and $\Omega$ notation means up to absolute constants.}.

\begin{theorem}[Sample complexity of PAC learning~\cite{zbMATH03391742,zbMATH04143473}]
\label{thm:BlumerPAC}
If $C \subseteq \{0,1\}^X$ has VC dimension $d$,
then $C$ is properly PAC learnable with 
$O((d\log (2/\eps)+\log(2/\delta))/\eps)$ samples,
generalization error $\eps$ and success probability $1-\delta$.
\end{theorem}

\subsection{Compression}
Littlestone and Warmuth~\cite{littleWarm} 
defined sample compression schemes as a natural abstraction
that captures a common property of many learning procedures,
like procedures for learning geometric shapes or algebraic structures
(see also~\cite{DBLP:conf/colt/Floyd89,DBLP:journals/ml/FloydW95}).

\paragraph{Definition.}
A sample compression scheme takes a long list of samples
and compresses it to a short sub-list of samples
in a way that allows to invert the compression.
Formally, a sample compression scheme for $C$ with kernel size $k$
and side information $I$, where $I$ is a finite set, consists of two maps $\kappa,\rho$ for which the following hold:

\begin{description}

\item[(${\kappa}$)] 
The {\em compression map}
$$\kappa: L_C(\infty) \to L_C(k) \times I$$ takes
$(Y,y)$ to $((Z,z),i)$ with $Z \subseteq Y$
and $z = y|_Z$.

\item[($\rho$)] The {\em reconstruction map}
$$\rho : L_C(k) \times I \to \{0,1\}^X$$
is so that for all $(Y,y)$ in $L_C(\infty)$,
$$\rho(\kappa(Y,y))|_Y = y.$$
\end{description}
The size of the scheme is\footnote{Logarithms in this text
are base $2$.} $k + \log (|I|)$.
In the language of coding theory,
the side information $I$
can be thought of as list decoding;
the map $\rho$ has a short list of possible reconstructions of a given $(Z,z)$,
and the information $i \in I$ indicates
which element in the list is the correct one.
See~\cite{DBLP:conf/colt/Floyd89,DBLP:journals/ml/FloydW95,MSWY15} for more discussions
of this definition, and some insightful examples.

\paragraph{Motivation and background.}
Littlestone and Warmuth 
showed that every compression scheme yields a natural learning procedure:
Given a labeled sample $(Y,y)$, the learner
compresses it to $\kappa(Y,y)$ and outputs the hypothesis $h = \rho(\kappa(Y,y))$.
They proved that this is indeed a PAC learner.

\begin{theorem}[Compression implies learnability \cite{littleWarm}]
\label{thm:LWPAC}
Let $C \subseteq \{0,1\}^X$, and
let $\kappa,\rho$ be a sample compression scheme
for $C$ of size $k$.
Let $d \geq 8 \big(k\log(2/\eps)+\log(1/\delta)\big)/\eps$. 
Then, the learning map $H: L_C(d) \to \{0,1\}^X$ 
defined by $H(Y,y) = \rho(\kappa(Y,y))$
is PAC learning $C$ with $d$ samples, 
generalization error $\eps$ and success probability $1-\delta$.
\end{theorem}

\begin{proof}[Proof sketch.]
Let $\mu$ be a distribution on $X$, and
$x_1,\ldots,x_d$ be $d$ independent samples from $\mu$.
There are $\sum_{j=0}^{k} {d \choose j}$
subsets $T$ of $[d]$ of size at most $k$.
There are $|I|$ choices for information $i \in I$.
Every fixing of $T,i$ yields a random function $h_{T,i}
= \rho((T,c|_T),i)$ that is measurable with respect to $x_T = (x_t : t \in T)$.
The random function $h_{T,i}$ is independent of $x_{[d]-T}$.
For every fixed $T,i,x_T$, 
therefore, if $\mu(\{x \in X: h_{T,i}(x) \neq c(x)\}) > \eps$
then the probability that $h_{T,i}$ agrees
with $c$ on all samples in $[d]-T$ is less than $(1-\eps)^{d-|T|}$.
The function $h$ is one of the functions in the random set $\{h_{T,i} : |T|\leq k,i\in I\}$,
and it satisfies $h|_Y = c|_Y$.
The union bound completes the proof.
\end{proof}

Littlestone and Warmuth also asked
whether the other direction holds:
{\it ``Are there concept classes with finite dimension for
which there is no scheme with bounded kernel size and bounded additional
information?''}

Further motivation for considering compression schemes comes from 
the problem of boosting a weak learner to a strong learner.
Boosting is a central theme in learning theory
that was initiated by Kearns and Valiant~\cite{Kearns88,DBLP:conf/stoc/KearnsV89}.
The boosting question, roughly speaking, is: given a learning algorithm with generalization error $0.49$, 
can we use it to get an algorithm with generalization error $\eps$ of our choice?
Theorem~\ref{thm:LWPAC} implies that if the learning algorithm
yields a sample compression scheme, then
boosting follows with a multiplicative overhead of roughly $1/\eps$
in the sample size.
In other words, efficient compression schemes immediately yield boosting.

Schapire~\cite{DBLP:journals/ml/Schapire90} and later on 
Freund~\cite{DBLP:journals/iandc/Freund95}
solved the boosting problem,
and showed how to efficiently boost the generalization error of PAC learners.
They showed that if $C$ is PAC learnable with $d$ samples
and generalization error $0.49$, then
$C$ is PAC learnable with $O(d \log^2(d/\eps) /\eps)$
samples and generalization error $\eps$
(see e.g.\ Corollary~3.3 in~\cite{DBLP:journals/iandc/Freund95}).
Interestingly, their boosting is based on a weak type of compression.
They showed  how to compress a sample of size $m$ to a sample of size roughly
$d \log m$, and that such compression already implies boosting
(see Section~\ref{sec:LiC} below for more details).

Additional motivation for studying sample compression schemes
relates to feature selection, which is about identifying meaningful
features of the underlying domain that are sufficient
for learning purposes (see e.g.~\cite{DBLP:journals/jmlr/GuyonE03}).
The existence of efficient compression schemes,
loosely speaking, shows that in any arbitrarily big data
there is a small set of features that already contains all the relevant information.
More concretely, a construction of an efficient compression scheme
provides tools that may be helpful for feature selection.


\paragraph{Previous constructions.}
Littlestone and Warmuth's question and variants of it lead to a rich body of work
that revealed profound properties of VC dimension and learning.
Floyd and Warmuth~\cite{DBLP:conf/colt/Floyd89,DBLP:journals/ml/FloydW95} 
constructed sample compression schemes of size $\log |C|$
for every finite concept class $C$.
They also constructed optimal compression schemes
of size $d$ for maximum classes\footnote{That is,
$C \subseteq \{0,1\}^X$ of
size $|C| = \sum_{j=0}^d {|X| \choose j}$ with $d = \VC(C)$.} of VC dimension $d$,
as a first step towards solving the general question.
As the study of sample compression schemes deepened,
many insightful and optimal schemes for special cases have been constructed:
Floyd~\cite{DBLP:conf/colt/Floyd89},
Helmbold et al.~\cite{DBLP:journals/siamcomp/HelmboldSW92},
Floyd and Warmuth~\cite{DBLP:journals/ml/FloydW95},
Ben-David and Litman~\cite{DBLP:journals/dam/Ben-DavidL98},
Chernikov and Simon~\cite{chernikovS},
Kuzmin and Warmuth~\cite{DBLP:journals/jmlr/KuzminW07},
Rubinstein et al.~\cite{DBLP:journals/jcss/RubinsteinBR09}, Rubinstein and Rubinstein~\cite{DBLP:journals/jmlr/RubinsteinR12},
Livni and Simon~\cite{DBLP:conf/colt/LivniS13} and more.
These works discovered and utilized connections
between sample compression schemes, and model theory,
topology, combinatorics, and geometry.
Finally, in our recent work with Shpilka and Wigderson~\cite{MSWY15},
we constructed sample compression schemes of size roughly 
$2^{O(d)} \cdot \log \log|C|$
for every finite concept class $C$ of VC dimension $d$.


\subsection{Our contribution}
\label{sec:LiC}

Our main theorem states that
VC classes have sample compression schemes of finite size.
The key property of this compression 
is that its size does not depend on the size of the given sample $(Y,y)$.

\begin{theorem}[Compression]
\label{thm:compressionVC}
If $C\subseteq \{0,1\}^X$ has VC dimension $d$,
then $C$ has a sample compression scheme of size $2^{O(d)}$.
\end{theorem}

Our construction (see~Section~\ref{sec:const}) of sample compression schemes 
is overall quite short and simple.
It is inspired by Freund's work~\cite{DBLP:journals/iandc/Freund95}
where majority is used to boost the accuracy of learning procedures.
It also uses several known properties of PAC learnability and VC dimension,
together with von Neumann's minimax theorem,
and it reveals approximate but efficient
equilibrium strategies for zero-sum games of low VC dimension
(see Section~\ref{sec:press} below).

The construction is even more efficient when the dual
class is also under control.
The dual concept class $C^*\subseteq \{0,1\}^C$ of $C$ is defined as
the set of all functions $f_x:C\rightarrow \{0,1\}$ 
defined by $f_x(c) = c(x)$.
If we think of $C$ as a binary matrix whose rows
are concepts in $C$ and columns are elements of $X$, then
$C^*$ corresponds to the distinct rows of the transposed matrix.

\begin{theorem}[Compression using dual VC dimension]
\label{thm:mainVC*}
If $C\subseteq \{0,1\}^X$ has VC dimension $d >0$
and $C^*$ has VC dimension $d^*>0$, 
then $C$ has a sample compression scheme of size $k \log k$
with $k = O(d^* \cdot d)$.
\end{theorem}

Theorem~\ref{thm:compressionVC} follows from Theorem~\ref{thm:mainVC*}
via the following bound, which was observed by Assouad~\cite{Assouad}.

\begin{claim}[Dual VC dimension \cite{Assouad}]
\label{clm:assou}
If $\VC(C) \leq d$, then $\VC(C^*) < 2^{d+1}$.
\end{claim}

A natural example for which the dual class is well behaved
is geometrically defined classes.
Assume, for example, that $C$ 
represents the incidence relation among halfspaces 
and points in $r$-dimensional real space
(a.k.a.\ sign rank or Dudely dimension $r$).
That is, for every $c \in C$ there is a vector $a_c \in \R^r$
and for every $x \in X$ there is a vector $b_x \in \R^r$ so that
$c(x) = 1$ if and only if the inner product 
$\langle a_c, b_x \rangle = \sum_{j=1}^r a_c(j) b_x(j)$
is positive.
It follows that $\VC(C) \leq r$,
but the symmetric structure also implies that $\VC(C^*) \leq r$.
So, the compression scheme constructed here for this $C$ actually 
has size $O(r^2 \log r)$ and not $2^{O(r)}$.

\paragraph{Proof background and overview.}
Freund~\cite{DBLP:journals/iandc/Freund95} and later on Freund and 
Schapire~\cite{DBLP:journals/jcss/FreundS97} showed that
for every class $C$ that is PAC learnable with $d$ samples, 
there exists a compression scheme that compresses
a $C$-labeled sample $(Y,y)$ of size $m$ to a sub-sample of size $k = O(d\log m)$
with additional information of $k \log k$ bits
(for a more detailed discussion, see 
Sections 1.2 and 13.1.5 in~\cite{schapire2012boosting}). 
Their constructive proof is iterative:
In each iteration $t$, a distribution $\mu_t$ on $Y$
is carefully and adaptively chosen. Then, $d$ independent points from $Y$ are drawn
according to $\mu_t$, and fed into the learning map to produce an hypothesis $h_t$.
They showed that after $T=O(\log(1/\eps))$ iterations, the majority vote $h$ 
over $h_1,\ldots,h_T$ is an $\eps$-approximation of $y$
with respect to the uniform measure on $Y$.
In particular, if we choose $\eps < 1/m$, then $h$ completely agrees with $y$ on $Y$.
This makes $T = O(\log m)$ and gives a sample compression scheme from 
a sample of size $m$ to a sub-sample of size $d \cdot T = O(d\log m)$.

The size of Freund and Schapire's compression scheme is not uniformly 
bounded, it depends on $|Y|$.
A first step towards removing this dependence
is observing that their proof can be replaced by a combination of
von Neumann's minimax theorem and a Chernoff bound.
In this argument, the $\log m$ factor eventually comes from a union bound
over the $m$ samples.
The compression scheme presented in this text replaces
the union bound with a more accurate analysis 
that utilizes the VC dimension of the dual class.
This analysis ultimately replaces the $\log m$ factor by a $d^*$ factor.

\section{Preliminaries}
\label{sec:press}

\paragraph{Approximations.}
The following theorem shows that every distribution can be approximated
by a distribution of small support, when
the statistical tests belong to a class of small VC dimension.
This phenomenon was first proved by Vapnik and Chervonenkis~\cite{zbMATH03391742}, 
and was later quantitively improved in~\cite{Li:2000:IBS:338219.338267,zbMATH00567173}.

\begin{theorem}[Approximations for bounded VC dimension \cite{zbMATH03391742,Li:2000:IBS:338219.338267,zbMATH00567173}]
\label{thm:VC}
Let $C\subseteq\{0,1\}^X$ of VC dimension $d$. Let $\mu$
be a distribution on $X$.
For all $\eps >0$, there exists a multiset $Y\subseteq X$ of size $|Y|\leq O(d/ \eps^2)$
such that for all $c\in C$,
$$\left| \mu(\{x \in X : c(x) =1 \}) - \frac{|\{x \in Y : c(x) =1 \}|}{|Y|} \right| \leq\eps. $$
\end{theorem}

\paragraph{Carath\'{e}odory's theorem.} 
The following simple lemma
can be thought of as an approximate and combinatorial
version of Carath\'{e}odory's theorem from convex geometry.
Let $C\subseteq\{0,1\}^n \subset \R^n$ and denote by $K$ the convex hull
of $C$ in $\R^n$. 
Carath\'{e}odory's theorem says that 
every point $p\in K$ is a convex combination of at most $n+1$
points from $C$. 
The lemma says that if $\VC(C^*)$ is small
then every $p \in K$ can be approximated by a convex combination with a small support.

\begin{lemma}[Sampling for dual VC dimension]
\label{lem:VCsample}
Let $C \subseteq \{0,1\}^X$
and let $d^* = \VC(C^*)$.
Let $p$ be a distribution on $C$ and let $\eps > 0$.
Then, 
$p$ can be $\eps$-approximated in $L^\infty$ by an average 
of at most $O(d^*/\eps^2)$ points from $C$.
That is, there is a multiset $F \subseteq C$ of size $|F| \leq O(d^*/\eps^2)$
so that for every $x \in X$,
$$\left| p ( \{ c \in C : c(x) = 1\}) - \frac{|\{ f \in F : f(x) = 1 \}|}{|F|} \right| 
\leq \eps .$$
\end{lemma}

\begin{proof}
Every $x \in X$ corresponds to a concept in $C^*$.
The distribution $p$ is a distribution on the domain of the
functions in $C^*$. 
The lemma follows by Theorem~\ref{thm:VC} applied to $C^*$.
\end{proof}

\paragraph{Minimax.}
Von Neumann's minimax theorem~\cite{Neumann1928}
is a seminal result in game theory (see e.g.\ the textbook \cite{owen1995game}).
Assume that there are 2 players\footnote{We focus on the case
of zero-sum games.}, a row player and a column player.
A pure strategy of the row player is $r \in [m]$
and a pure strategy of the column player is $j \in [n]$.
A mixed strategy is a distribution on pure strategies.
Let $M$ be a binary matrix so that
$M(r,j) = 1$ if and only if the row player wins the game
when the pure strategies $r,j$ are played.

The minimax theorem says that if for every mixed strategy 
$q$ of the column player,
there is a mixed strategy $p$ of the row player
that guarantees that the row player wins with probability at least $V$,
then there is a mixed strategy $p^*$ of the row player
so that for all mixed strategies $q$ of the column player, 
the row player wins with probability at least $V$.
A similar statement holds for the column player.
This implies that there is
a pair of mixed strategies $p^*,q^*$ that form a Nash equilibrium
for the zero-sum game $M$ defines
(see \cite{owen1995game}).

\begin{theorem}[Minimax \cite{Neumann1928}]
\label{thm:minmax}
Let $M\in\mathbb{R}^{m\times n}$ be a real matrix.
Then,
$$\min_{p\in \Delta^m}\max_{q\in \Delta^n} \ p^tMq =
\max_{q\in \Delta^n}\min_{p\in\Delta^m} \ p^tMq,$$
where $\Delta^\ell$ is the set of distributions on $[\ell]$.
\end{theorem}

The arguments in the proof of Theorem~\ref{thm:mainVC*} below imply the following
variant of the minimax theorem, which may be of interest
in the context of game theory.
The minimax theorem holds for a general matrix $M$.
In other words, there is no assumption on the set of winning/losing states
in the game.

We observe that a combinatorial restriction on the winning/losing states
in the game implies that there is an approximate efficient equilibrium state.
Namely, if the rows of $M$ have VC dimension $d$ and the columns of $M$ have
VC dimension $d^*$, then for every $\eps>0$,
there is a multiset of $O(d^*/\eps^2)$ pure strategies $R \subseteq [m]$ for the row player,
and a multiset of $O(d/\eps^2)$ pure strategies $J \subseteq [n]$ for the column player,
so that a uniformly random choice from $R,J$ guarantees the players
a gain that is $\eps$-close to the gain in the equilibrium strategy.
Such a pair of mixed strategies is called an $\eps$-Nash equilibrium.
Lipton and Young~\cite{DBLP:journals/corr/cs-CC-0205035} 
showed that in every zero-sum game 
there are $\eps$-Nash equilibriums with logarithmic support\footnote{Lipton, Markakis and Mehta~\cite{Lipton:2003:PLG:779928.779933} proved a similar statement for general games.}. 
The ideas presented here show that if, say, the rows of the matrix of the game
have constant VC dimension, then there are $\eps$-Nash equilibriums with constant support.

\section{A sample compression scheme}
\label{sec:const}

We start with a high level description of the compression process
(Theorem~\ref{thm:mainVC*}).
Given a sample of the form $(Y,y)$,
the compression identifies $T \leq O(d^*)$
subsets $Z_1,\ldots,Z_T$ of $Y$, each of size at most $d$.
It then compresses $(Y,y)$
to $(Z,z)$ with $Z = \bigcup_{t \in [T]} Z_t$ and $z = y|_Z$.
The additional information $i \in I$ allows to recover $Z_1,\ldots,Z_T$ from $Z$.
The reconstruction process uses the information $i \in I$
to recover $Z_1,\ldots,Z_T$ from $Z$,
and then uses the PAC learning map $H$
to generate $T$ hypotheses $h_1,\ldots,h_T$ defined as $h_t = H(Z_t,z|_{Z_t})$.
The final reconstruction hypothesis $h = \rho((Z,z),i)$ is the majority vote
over $h_1,\ldots,h_T$.

\begin{proof}[Proof of Theorem~\ref{thm:mainVC*}]
Since the VC dimension of $C$ is $d$,
by Theorem~\ref{thm:BlumerPAC},
there is $s=O(d)$ and a proper learning map $H:L_C(s) \to C$ so that
for every $c \in C$ and for every probability distribution $q$ on $X$,
there is $Z \subseteq \text{supp}(q)$ of size $|Z| \leq s$
so that $q(\{x \in X : h_Z(x) \neq c(x)\}) \leq 1/3$
where $h_Z = H(Z,c|_Z)$.

\paragraph{Compression.} 
Let $(Y,y)\in L_C(\infty)$.
Let 
$${\cal H} = {\cal H}_{Y,y} = \{H(Z,z) : Z\subseteq Y, |Z|\leq s, z=y|_Z\}
\subseteq C.$$
The compression is based on the following claim.

\begin{claim}
\label{clm:ThereIsF}
There are $T \leq O(d^*)$ sets $Z_1,Z_2,\ldots,Z_T \subseteq Y$, 
each of size at most $s$, so that the following holds.
For $t \in [T]$, let
\begin{align}
\label{eqn:Ft}
f_t = H(Z_t,y|_{Z_t}).
\end{align}
Then, for every $x \in Y$,
\begin{align} 
\label{eqn:MajOfF}
|\{ t \in [T] : f_t(x) = y(x)\}| > T/2 .
\end{align}
\end{claim}

Given the claim, the compression $\kappa(Y,y)$ is defined as
$$Z = \bigcup_{t \in [T]} Z_t \ \ \text{and} \ \ z = y|_Z.$$
The additional information $i \in I$ allows to recover
the sets $Z_1,\ldots,Z_T$ from the set $Z$.
There are many possible ways to encode this information,
but the size of $I$ can be chosen to be at most $k^{k}$
with $k = 1+ O(d^*) \cdot s \leq O(d^*\cdot d)$.

\begin{proof}[Proof of Claim~\ref{clm:ThereIsF}]
By choice of $H$,
for every distribution $q$ on $Y$, there is  $h \in {\cal H}$ so that
$$q\left(\{x \in Y : h(x) = y(x) \} \right)\geq 2/3.$$
By Theorem~\ref{thm:minmax}, there is a distribution $p$
on ${\cal H}$ such that for every $x\in Y$,
\begin{align*}
p(\{h \in {\cal H} : h(x) = y(x)\}) \geq 2/3.
\end{align*}
By Lemma~\ref{lem:VCsample} applied to ${\cal H}$ and $p$
with $\eps=1/8$,
there is a multiset $F  = \{f_1,f_2,\ldots,f_T\} \subseteq {\cal H}$ of size $T \leq O(d^*)$ so that for every $x \in Y$,
\begin{align*} 
\frac{|\{ t \in [T] : f_t(x) = y(x) \}|}{T}
\geq p(\{h \in {\cal H} : h(x) = y(x) \}) - 1/8 > 1/2.
\end{align*}
For every $t \in [T]$, let
$Z_t$ be a subset of $Y$ of size $|Z_t| \leq d$ so that 
\begin{align*}
H(Z_t,y|_{Z_t}) = f_t.
\end{align*}
\end{proof}

\paragraph{Reconstruction.}
Given $((Z,z),i)$,
the information $i$ is interpreted as a list of $T$ subsets 
$Z_1,\ldots,Z_T$ of $Z$, each of size at most $d$.
For $t \in [T]$, let
$$h_t = H(Z_t,z|_{Z_t}).$$
Define $h = \rho((Z,z),i)$ as follows:
For every $x \in X$, let
$h(x)$ be a symbol in $\{0,1\}$ that appears most in the list 
$$\lambda_x((Z,z),i) = (h_1(x),h_2(x),\ldots,h_T(x)),$$
where ties are arbitrarily broken.

\paragraph{Correctness.} 
Fix $(Y,y) \in L_C(\infty)$.
Let $((Z,z),i) = \kappa(Y,y)$ and $h = \rho((Z,z),i)$.
For $x \in Y$, consider the list
$$\phi_x(Y,y) = (f_1(x),f_2(x),\ldots,f_T(x))$$
defined in the compression process of $(Y,y)$.
The list $\phi_x(Y,y)$ is identical to the list $\lambda_x((Z,z),i)$
due to the following three reasons:
Equation~\eqref{eqn:Ft}; 
the information $i$ allows to correctly recover
$Z_1,\ldots,Z_T$; 
and $y|_{Z_t} = z|_{Z_t}$ for all $t \in [T]$.
Finally, by~\eqref{eqn:MajOfF}, for every $x \in Y$, the symbol
$y(x)$ appears in more than half of the list $\lambda_x((Z,z),i)$ so indeed $h(x) = y(x)$.
\end{proof}

\section{Concluding remarks and questions}

We have shown that every VC class admits
a sample compression scheme with size 
exponential in its VC dimension.
This is the first bound that depends only on the VC dimension,
and holds for all binary-labeled classes.
It is worth noting that
many of the known compression schemes 
for special cases, like~\cite{DBLP:journals/ml/FloydW95,
DBLP:journals/dam/Ben-DavidL98,
DBLP:journals/jmlr/KuzminW07,
DBLP:journals/jmlr/RubinsteinR12,
DBLP:conf/colt/LivniS13},
have size $d$ or $O(d)$ which is essentially optimal.
In many of these cases, 
our construction is in fact of size polynomial in $d$, since the VC dimension of the dual class is small as well. 
Nevertheless, 
Floyd and Warmuth's question~\cite{DBLP:journals/ml/FloydW95,DBLP:conf/colt/Warmuth03} whether 
sample compression schemes of size $O(d)$ always exist remains open.

\paragraph{Multi-labeled classes.}
Unlike VC dimension, sample compression schemes
as well as the fact that they imply PAC learnability naturally
generalizes to multi-labeled concept classes
(see e.g.~\cite{DBLP:conf/alt/SameiYZ14}.) 
Littlestone and Warmuth's question
is therefore an instance of a more general question:
Does the size of an optimal sample compression scheme
for a given class capture the sample complexity of PAC learning of this class?
A positive answer to this question
will yield a universal and natural parameter that captures 
the sample complexity of PAC learning.

There are many generalization of  VC dimension to multi-labeled 
concept classes $C\subseteq \Sigma^X$,
see \cite{BenDavid95} and references within.
An example that naturally comes up in our analysis is the distinguishing dimension $\DD(C)$: For every $c \in C$, define a binary concept class
$B_c \subseteq \{0,1\}^X$ as the set of all $b_h$, for $h \in C$,
defined by $b_h(x) = 1$ if and only if $h(x) = c(x)$.
Define
$$\DD(C) = \sup \{ \VC(B_{c}) : c \in C \}.$$
If $C$ is binary then $\VC(C) = \DD(C)$.
This definition of dimension is similar to notions used
in \cite{Natarajan89,Dudley87,BenDavid95}.
It can be verifies that if $C$ is multi-labeled then
our compression scheme for $C$ has size exponential in $\DD(C)$.
However, although $\Omega(\VC(C))$ is a lower bound on the sample
complexity of PAC learning for a binary-labeled $C$,
the distinguishing dimension $\DD(C)$
is not a lower bound on the sample complexity
of PAC learning for a multi-labeled $C$.
Indeed, an example constructed by
Danieli and Shalev-Schwartz~\cite{DBLP:conf/colt/DanielyS14}
implies that there is a concept class $C \subseteq \Sigma^X$
that is properly PAC learnable with $O(1)$ samples
but $\DD(C) \geq \Omega(\log|\Sigma|)$.

\paragraph{Learners' complexity.}
The efficiency of our construction relies on 
the fact that every binary-labeled concept class $C$ has a proper learner
with optimal sample complexity. A closer look at the proof 
reveals that it is valid
even if the learner is not proper; it suffices that the set of hypotheses produced
by the learner have low VC dimension.

This motivates the following natural question:
Is it true that for every learning map $H$ for $C \subseteq \{0,1\}^X$ with $\VC(C)=d$
and for every $c \in C$,
the set of hypotheses that $H$ outputs 
when learning $c$ has VC dimension $O(d)$ as well?

The answer is negative;
some students learn 
although they make things more complicated than necessary.
Here is an example.
Let $n$ be a power of $2$, and
consider the concept class $C  = \{(00\ldots 0)\} \subset \{0,1\}^{X}$ with
$X=[n+3 \log n]$
consisting only of the all zero concept.
The learning map $H$ gets as input a labeled sample $(Y,y) \in L_C(3)$ of size $3$,
and outputs the following hypothesis $h$.
If $Y \not \subseteq [n]$ then $h$ is defined to be $0$ everywhere.
Otherwise,
$h$ is defined as $0$ on $[n]$
and on the last $3 \log n$ coordinates $h$ is defined as $\psi(Y)$,
where $\psi$ is a bijection from $[n]^3$ to $\{0,1\}^{[3 \log n]}$.
First, the image of $H$ has VC dimension $3 \log n$
since the last $3 \log n$ coordinates are shattered by it.
Second, the map $H$ is a PAC learner for $C$. 
Indeed, let $\mu$ be a distribution on $X$.
If $\mu([n]) \geq 2/3$ then the error of $h$ is always smaller than $1/3$.
If $\mu([n]) < 2/3$ then the only case that
$h$ has positive error is that $Y \subseteq [n]$,
which happens with probability $(2/3)^3 < 1/3$.

A variation of the question above is:
Does every multi-labeled class $C$
have a learner $H$ that makes a nearly optimal number of samples
with an image that is not much more complicated than $C$?

The answer for binary-labeled classes is affirmative;
$C$ has a nearly optimal proper learner.
Danieli and Shalev-Schwartz~\cite{DBLP:conf/colt/DanielyS14}
showed that there are multi-labeled concept classes
that are PAC learnable with $O(1)$ samples
but are not properly PAC learnable with $O(1)$ samples.
In their example, however, the image of $H$ has just one more concept than $C$.
This question therefore remains open.

\section*{Acknowledgements}

We thank Amir Shpilka and Avi Wigderson for helpful discussions.
We also thank Ben Lee Volk and Manfred Warmuth for comments on an earlier version of this text.

\bibliographystyle{plain} 

\bibliography{compRef}

\begin{thebibliography}{10}

\bibitem{Assouad}
P.~Assouad.
\newblock Densite et dimension.
\newblock {\em Ann. Institut Fourter}, 3:232--282, 1983.

\bibitem{BenDavid95}
S.~Ben{-}David, N.~Cesa{-}Bianchi, D.~Haussler, and P.~M. Long.
\newblock Characterizations of learnability for classes of \{0,...,n\}-valued
  functions.
\newblock {\em J. Comput. Syst. Sci.}, 50(1):74--86, 1995.

\bibitem{DBLP:journals/dam/Ben-DavidL98}
S.~Ben{-}David and A.~Litman.
\newblock Combinatorial variability of {Vapnik-Chervonenkis} classes with
  applications to sample compression schemes.
\newblock {\em Discrete Applied Mathematics}, 86(1):3--25, 1998.

\bibitem{zbMATH04143473}
A.~{Blumer}, A.~{Ehrenfeucht}, D.~{Haussler}, and M.~K. {Warmuth}.
\newblock {Learnability and the Vapnik-Chervonenkis dimension.}
\newblock {\em {J. Assoc. Comput. Mach.}}, 36(4):929--965, 1989.

\bibitem{chernikovS}
A.~Chernikov and P.~Simon.
\newblock Externally definable sets and dependent pairs.
\newblock {\em Israel Journal of Mathematics}, 194(1):409--425, 2013.

\bibitem{Cristianini00a}
N.~Cristianini and J.~Shawe-Taylor.
\newblock {\em An Introduction to Support Vector Machines and other
  kernel-based learning methods}.
\newblock Cambridge University Press, 2000.

\bibitem{DBLP:conf/colt/DanielyS14}
A.~Daniely and S.~Shalev{-}Shwartz.
\newblock Optimal learners for multiclass problems.
\newblock In {\em COLT}, pages 287--316, 2014.

\bibitem{Dudley87}
R.~M. Dudley.
\newblock Universal {D}onsker classes and metric entropy.
\newblock {\em Ann. Probab.}, 15(4):1306--1326, 10 1987.

\bibitem{DBLP:conf/colt/Floyd89}
S.~Floyd.
\newblock Space-bounded learning and the vapnik-chervonenkis dimension.
\newblock In {\em COLT}, pages 349--364, 1989.

\bibitem{DBLP:journals/ml/FloydW95}
S.~Floyd and M.~K. Warmuth.
\newblock Sample compression, learnability, and the vapnik-chervonenkis
  dimension.
\newblock {\em Machine Learning}, 21(3):269--304, 1995.

\bibitem{DBLP:journals/iandc/Freund95}
Y.~Freund.
\newblock Boosting a weak learning algorithm by majority.
\newblock {\em Inf. Comput.}, 121(2):256--285, 1995.

\bibitem{DBLP:journals/jcss/FreundS97}
Y.~Freund and R.~E. Schapire.
\newblock A decision-theoretic generalization of on-line learning and an
  application to boosting.
\newblock {\em J. Comput. Syst. Sci.}, 55(1):119--139, 1997.

\bibitem{schapire2012boosting}
Y.~Freund and R.~E. Schapire.
\newblock {\em Boosting: Foundations and Algorithms}.
\newblock Adaptive computation and machine learning. MIT Press, 2012.

\bibitem{DBLP:journals/jmlr/GuyonE03}
I.~Guyon and A.~Elisseeff.
\newblock An introduction to variable and feature selection.
\newblock {\em Journal of Machine Learning Research}, 3:1157--1182, 2003.

\bibitem{DBLP:journals/siamcomp/HelmboldSW92}
D.~P. Helmbold, R.~H. Sloan, and M.~K. Warmuth.
\newblock Learning integer lattices.
\newblock {\em {SIAM} J. Comput.}, 21(2):240--266, 1992.

\bibitem{Kearns88}
M.~Kearns.
\newblock Thoughts on hypothesis boosting.
\newblock {\em Unpublished manuscript}, 1988.

\bibitem{DBLP:conf/stoc/KearnsV89}
M.~Kearns and L.~G. Valiant.
\newblock Cryptographic limitations on learning boolean formulae and finite
  automata.
\newblock In David~S. Johnson, editor, {\em STOC}, pages 433--444. {ACM}, 1989.

\bibitem{KearnsVazirani94}
M.~Kearns and U.~V. Vazirani.
\newblock {\em An introduction to computational learning theory}.
\newblock MIT Press, Cambridge, MA, USA, 1994.

\bibitem{DBLP:journals/jmlr/KuzminW07}
D.~Kuzmin and M.~K. Warmuth.
\newblock Unlabeled compression schemes for maximum classes.
\newblock {\em Journal of Machine Learning Research}, 8:2047--2081, 2007.

\bibitem{Li:2000:IBS:338219.338267}
Y.~Li, P.~M. Long, and A.~Srinivasan.
\newblock Improved bounds on the sample complexity of learning.
\newblock In {\em SODA}, pages 309--318, 2000.

\bibitem{Lipton:2003:PLG:779928.779933}
R.~J. Lipton, E.~Markakis, and A.~Mehta.
\newblock Playing large games using simple strategies.
\newblock In {\em ACM Conference on Electronic Commerce}, pages 36--41, New
  York, NY, USA, 2003. ACM.

\bibitem{DBLP:journals/corr/cs-CC-0205035}
R.~J. Lipton and N.~E. Young.
\newblock Simple strategies for large zero-sum games with applications to
  complexity theory.
\newblock {\em CoRR}, cs.CC/0205035, 2002.

\bibitem{littleWarm}
N.~Littlestone and M.~Warmuth.
\newblock Relating data compression and learnability.
\newblock {\em Unpublished}, 1986.

\bibitem{DBLP:conf/colt/LivniS13}
R.~Livni and P.~Simon.
\newblock Honest compressions and their application to compression schemes.
\newblock In {\em COLT}, pages 77--92, 2013.

\bibitem{MSWY15}
S.~Moran, A.~Shpilka, A.~Wigderson, and A.~Yehudayoff.
\newblock Teaching and compressing for low {VC}-dimension.
\newblock {\em ECCC}, TR15-025, 2015.

\bibitem{Natarajan89}
B.~K. Natarajan.
\newblock On learning sets and functions.
\newblock {\em Machine Learning}, 4:67--97, 1989.

\bibitem{Neumann1928}
J.~von Neumann.
\newblock Zur theorie der gesellschaftsspiele.
\newblock {\em Mathematische Annalen}, 100:295--320, 1928.

\bibitem{owen1995game}
G.~Owen.
\newblock {\em Game Theory}.
\newblock Academic Press, 1995.

\bibitem{DBLP:journals/jcss/RubinsteinBR09}
B.~I.~P. Rubinstein, P.~L. Bartlett, and J.~H. Rubinstein.
\newblock Shifting: One-inclusion mistake bounds and sample compression.
\newblock {\em J. Comput. Syst. Sci.}, 75(1):37--59, 2009.

\bibitem{DBLP:journals/jmlr/RubinsteinR12}
B.~I.~P. Rubinstein and J.~H. Rubinstein.
\newblock A geometric approach to sample compression.
\newblock {\em Journal of Machine Learning Research}, 13:1221--1261, 2012.

\bibitem{DBLP:conf/alt/SameiYZ14}
R.~Samei, B.~Yang, and S.~Zilles.
\newblock Generalizing labeled and unlabeled sample compression to multi-label
  concept classes.
\newblock In {\em {ALT}}, pages 275--290, 2014.

\bibitem{DBLP:journals/ml/Schapire90}
R.~E. Schapire.
\newblock The strength of weak learnability.
\newblock {\em Machine Learning}, 5:197--227, 1990.

\bibitem{zbMATH00567173}
M.~{Talagrand}.
\newblock {Sharper bounds for Gaussian and empirical processes.}
\newblock {\em {Ann. Probab.}}, 22(1):28--76, 1994.

\bibitem{zbMATH03943062}
L.G. {Valiant}.
\newblock {A theory of the learnable.}
\newblock {\em {Commun. ACM}}, 27:1134--1142, 1984.

\bibitem{zbMATH03391742}
V.N. {Vapnik} and A.Ya. {Chervonenkis}.
\newblock {On the uniform convergence of relative frequencies of events to
  their probabilities.}
\newblock {\em {Theory Probab. Appl.}}, 16:264--280, 1971.

\bibitem{DBLP:conf/colt/Warmuth03}
M.~K. Warmuth.
\newblock Compressing to {VC} dimension many points.
\newblock In {\em COLT/Kernel}, pages 743--744, 2003.

\end{thebibliography}

\end{document}